\documentclass[lettersize,journal]{IEEEtran}
\usepackage{amsmath,amsfonts}
\usepackage{algorithmic}
\usepackage{algorithm}
\usepackage{array}
\usepackage{textcomp}
\usepackage{stfloats}
\usepackage{url}
\usepackage{verbatim}
\usepackage{graphicx}
\usepackage{cite}
\usepackage{bibentry}
\usepackage{xcolor} 
\usepackage{enumitem}
\usepackage{amsthm,amssymb,graphicx,multirow,amsmath,color,amsfonts}
\usepackage{subcaption}

\usepackage{tabularx}

\newtheorem{corollary}{Corollary}
\newtheorem{theorem}{\bf Theorem}

\newtheorem{proposition}{\bf Proposition}

\newtheorem{definition}{\bf Definition}

\begin{document}

\title{Brainstorming Generative Adversarial Networks (BGANs):\\
	Towards  Multi-Agent Generative Models with Distributed Datasets}

\author{Aidin~Ferdowsi, Walid~Saad,~\IEEEmembership{Fellow,~IEEE,}

\IEEEcompsocitemizethanks{\IEEEcompsocthanksitem Authors are with the Bradley Department of Electrical and Computer Engineering, Virginia Tech, Blacksburg, VA, USA, 24060.\protect\\
	E-mails:\{aidin,walids\}@vt.edu}}

\maketitle

\begin{abstract}
To achieve a high learning accuracy, generative adversarial networks (GANs) must be fed by large datasets that adequately represent the data space. However, in many scenarios, the available datasets may be limited and distributed across multiple agents, each of which is seeking to learn the distribution of the data on its own. In such scenarios, the agents often do not wish to share their local data as it can cause communication overhead for large datasets. In this paper, to address this multi-agent GAN problem, a novel brainstorming GAN (BGAN) architecture is proposed using which multiple agents can generate real-like data samples while operating in a fully distributed manner. BGAN allows the agents to gain information from other agents without sharing their real datasets but by ``brainstorming'' via the sharing of their generated data samples. In contrast to existing distributed GAN solutions, the proposed BGAN architecture is designed to be fully distributed, and it does not need any centralized controller. Moreover, BGANs are shown to be scalable and not dependent on the hyperparameters of the agents' deep neural networks (DNNs) thus enabling the agents to have different DNN architectures. Theoretically, the interactions between BGAN agents are analyzed as a game whose unique Nash equilibrium is derived. Experimental results show that BGAN can generate real-like data samples with higher quality and lower Jensen-Shannon divergence (JSD) and Fr\`echet Inception distance (FID) compared to other distributed GAN architectures.

\end{abstract}

\section{Introduction}
\IEEEPARstart{G}{enerative} adversarial networks (GANs) are deep neural network (DNN) architectures that can learn a dataset distribution and generate realistic data points similar to this dataset \cite{NIPS2014_5423}. In GANs, a DNN called \emph{generator} generates data samples while another DNN called \emph{discriminator} tries to discriminate between the generator's data and the actual data. The interaction between the generator and discriminator results in optimizing the DNN weights such that the generator's generated samples look similar to the realistic data. Recently, GANs were adopted in several applications such as image synthesis \cite{li2016precomputed}, anomaly detection \cite{ferdowsi2019generative}, text to image translation \cite{reed2016generative}, speech processing \cite{pascual2017segan}, and video generation \cite{vondrick2016generating}.

Similar to many deep learning algorithms, GANs require large datasets to execute their associated tasks. Conventionally, such datasets are collected from the end-users of an application and stored at a data center to be then used by a central workstation or cloud to learn a task. However, relying on a central workstation requires powerful computational capabilities and can cause large delays. On the other hand, such central data storage units are vulnerable to external attacks. Furthermore, in many scenarios such as health and financial applications, the datasets are distributed across multiple agents (e.g., end-users) who do not intend to share them. Such challenges motivate parallelism and the need for distributed, multi-agent learning for GANs.

In a distributed learning architecture, \emph{multiple agents} can potentially learn the GAN task in a decentralized fashion by sharing some sort of information with each other while minimizing the communication overhead. The goal of each agent in a distributed GAN would be to learn how to generate high quality real-like data samples. Distributed GAN learning schemes also can also reduce the communication and computational limitations of centralized GAN models making them more practical for large-scale scenarios with many agents.

\subsection{Related Works}
For deep learning models, several distributed architectures have been proposed to facilitate parallelism using multiple computational units\cite{2008mapreduce,low2012distributed,pmlr,NIPS2017_6749,dean2012large,konevcny2016federated}. The authors in \cite{2008mapreduce} introduced the MapReduce architecture, in different agents aim at mapping the data into a new space reducing the data size. Moreover, GraphLab abstraction was proposed in \cite{low2012distributed} to facilitate graph computation across multiple workstations. In \cite{pmlr}, the authors have proposed a stochastic gradient push for distributed deep learning that converges to a stationary point. In addition, in \cite{NIPS2017_6749}, the authors have proposed to use ternary gradients to accelerate distributed deep learning in data parallelism. Furthermore, model parallelism is introduced in \cite{dean2012large}, in which the different layers and weights of a DNN structure are distributed between several agents. Recently, federated learning (FL) was introduced in \cite{konevcny2016federated} as an effective distributed learning mechanism that allows multiple agents to train a global model independently on their dataset and communicate training updates to a central server that aggregates the agent-side updates to train the global model. However, the works in \cite{2008mapreduce,low2012distributed,pmlr,NIPS2017_6749,dean2012large} and \cite{konevcny2016federated} as well as follow ups on FL focus on inference models and do not deal with generative models or GAN.

Recently, in \cite{hoang2017multi,durugkar2016generative,ghosh2017multi,hardy2019md}, and \cite{yonetani2019decentralized} the authors investigated the use of  distributed architectures that take into account GAN's unique structure which contains two separate DNNs (generator and discriminator). In \cite{hoang2017multi} and \cite{durugkar2016generative}, multiple generators or discriminators are used to stabilize the learning process but not to learn from multiple datasets. In \cite{ghosh2017multi}, a single discriminator is connected to multiple generators in order to learn multiple modalities of a dataset and to address the mode collapse problem. In \cite{hardy2019md}, the notion of privacy preserving GAN agents was studied for the first time, using two architectures: a) A multi-discriminator GAN (MDGAN) which contains multiple discriminators each located at every agent that owns private data and a central generator that generates the data and communicates its to each agent, and b) an adaptation of FL called FLGAN in which every agent trains a global GAN on its own data using a single per-agent discriminator and a per-agent generator and communicates the training updates to a central aggregator that learns a global GAN model. In \cite{yonetani2019decentralized}, analogously to MDGAN, a forgiver-first update (F2U) GAN is proposed such that every agent owns a discriminator and a central node has a generator. However, unlike the MDGAN model, at each training step, the generator's parameters are updated using the output of the most forgiving discriminator. Note that in the context of distributed learning privacy is defined as the right of controlling your own data and not sharing it with others. This is aligned with the common privacy definition of FL \cite{konevcny2016federated}. 

However, the works in \cite{2008mapreduce,low2012distributed,NIPS2017_6749,dean2012large,hoang2017multi}, and \cite{durugkar2016generative} consider a centralized dataset accessed by all of the agents. Moreover, the solutions in \cite{konevcny2016federated} and \cite{hardy2019md} can cause communication overhead since they require the agents to communicate, at every iteration, the DNN trainable parameters to a central node. Also, none of the GAN architectures in \cite{hoang2017multi,durugkar2016generative,ghosh2017multi,hardy2019md}, and \cite{yonetani2019decentralized} is fully distributed and they all require either a central generator or a central discriminator. In addition, the distributed GAN solutions in \cite{hoang2017multi,durugkar2016generative,ghosh2017multi,hardy2019md}, and \cite{yonetani2019decentralized} do not consider heterogeneous computation and storage capabilities for the agents.

\subsection{Contributions}
{The main contribution of this paper is the introduction of a novel brainstorming GAN (BGAN) architecture. This architecture allows agents to learn a data distribution in a fashion that is fully distributed, eliminating the need for a central controller.}  {In the BGAN architecture, each agent, equipped with a single generator and a single discriminator, owns a unique dataset. During training, agents share their \emph{ideas}—generated data samples—with their neighbors. This allows them to communicate information about their dataset without sharing the actual data samples.} As such, the proposed approach enables the GANs to collaboratively brainstorm in order to generate high quality real-like data samples. This is analogous to how humans \emph{brainstorm ideas} to come up with solutions. {To the best of our knowledge, this is the first work that proposes a multi-agent learning architecture for GANs that operates without a central controller, as mentioned earlier.} In particular, the proposed BGAN has the following key features:
\begin{enumerate}
	\item The BGAN architecture is fully distributed and does not require any centralized controller.
	\item {Compared to previous distributed GAN models such as MDGAN, FLGAN, and F2U, our approach significantly reduces communication overhead.}
	\item It allows defining different DNN architectures for different agents depending on their computational and storage capabilities.
\end{enumerate}
To characterize the performance of BGAN, we define a game between the BGAN agents and we analytically derive its Nash equilibrium (NE). We prove the uniqueness of the derived NE for the defined game. Moreover, we analyze each agent's connection structure with neighboring agents and characterize the minimum connectivity requirements that enable each agent to gain information from all of the other agents. We compare the performance of our proposed BGAN with other state-of-the-art architectures such as MDGAN, FLGAN, and F2U and show that BGAN can outperform them in terms of Jensen-Shannon divergence (JSD), Fr\`echet Inception distance (FID), and communication requirements besides the fact that, unlike the other models, BGAN is fully distributed and allows different DNN architectures for agents. 

{BGANs are particularly useful in practical, real-world scenarios where data owners wish to receive information from other data owners without sharing their own data.} {For instance, hospitals that own patients’ data, such as radiology images of patient lungs, may wish to generate synthetic images for future experiments. To improve the image generation process, they can use BGAN to receive information from other hospitals without sharing their own datasets.} {Similarly, financial firms, which often have scarce datasets that do not cover the entire data space, can use BGANs to learn better data representations without sharing their local data.} Another key application is within the context of an Internet of Things. In particular, Internet of things devices collect information from their users (e.g., location, voice, heartbeats) that cannot be shared publicly, thus they can use BGANs for information flow and to learn better data space representations. Therefore, BGANs can be used in broad range of areas such as medical applications, finance, personal activity, and the Internet of things.  

The rest of the paper is organized as follows. Section \ref{sec:system} describes the multi-agent learning system model. In Section \ref{sec:BGAN}, the BGAN architecture is proposed and the analytical results are derived. Experimental results are presented in Section \ref{sec:exp} and conclusions are drawn in Section \ref{sec:conc}. 

\section{System Model}\label{sec:system}
Consider a set $ \mathcal{N} $ of $ n $ agents, each represented as a node in Figure \ref{fig:BGAN}. Every agent $ i $ owns a unique dataset $ \mathcal{D}_i $ which follows a distribution $ p_{\text{data}_i} $. We define a set $\mathcal{D}$, represented as the outer circle in Figure \ref{fig:BGAN}, which is the union of all datasets $\mathcal{D}_1$, $\mathcal{D}_2$, $\dots$, $\mathcal{D}_n$. This total available data follows a distribution $ p_{\text{data}} $, which is the target distribution that our model aims to mimic. {It is worth noting that the connectivity pattern shown in Figure \ref{fig:BGAN} is a simplified representation for illustrative purposes. In practice, the connectivity can be any configuration as deemed suitable for the application.} For each agent $ i $, $ p_{\text{data}_i} $ is a data distribution that does not span the entire data space as good as $ p_\text{data} $. In this model, every agent $ i  $ tries to learn a \emph{generator} distribution $ p_{g_i} $ over its available dataset $ \mathcal{D}_i $ that can be close as possible to $p_{\text{data}}  $. To learn $ p_{g_i} $ at every agent $ i \in \mathcal{N} $, we define a prior input noise $ z $ with distribution $ p_{z_i}(z) $ and a mapping $G_i(z,\boldsymbol{\theta}_{g_i}) $ from this random variable $ z $ to the data space, where $ G_i $ is a DNN with a vector of parameters $ \boldsymbol{\theta}_{g_i} $. For every agent $i \in \mathcal{N}$, we also define another DNN called \emph{discriminator} $ D_i(\boldsymbol{x},\boldsymbol{\theta}_{d_i}) $ with vector of parameters {$ \boldsymbol{\theta}_{d_i} $} that gets a data sample $ \boldsymbol{x} $ as an input and outputs a value between $ 0 $ and $ 1 $. When the output of the discriminator is closer to $ 1 $, then the received data sample is deemed to be real and when the output is closer to $ 0 $ it means the received data is fake. In our BGAN's architecture, the goal is to find the distribution of the total data, $ p_\text{data} $ under the constraint that no agent $ i $ will share its available dataset $ \mathcal{D}_i$ and DNN parameters $ \boldsymbol{\theta}_{d_i} $ and $ \boldsymbol{\theta}_{g_i} $ with other agents. In contrast to MDGAN, FLGAN, and F2U, we will propose an architecture which is fully distributed and does not require any central controller. In BGANs, the agents only share their \emph{ideas} about the data distribution with an idea being defined as the output of  $ G_i(z,\boldsymbol{\theta}_{g_i})  $, at every epoch of the training phase with the other agents. Then, they will use the shared ideas to brainstorm and collaboratively learn the required data distribution.

While every agent's generator DNN tries to generate data samples close to the real data, the discriminator at every agent aims at discriminating the fake data samples from the real data samples that it owns. Hence, we model these interactions between the generators and discriminators of the agents by a game-theoretic framework. For a standalone agent $ i $ that does not communicate with other agents, one can define a zero-sum game between its generator and discriminator such that its local \emph{value} function is \cite{NIPS2014_5423}:
\begin{align}
	\tilde{V}_i (D_i,G_i) &= \mathbb{E}_{\boldsymbol{x}\sim p_{\textrm{data}_i}}\left[\log D_i(x)\right] \nonumber \\&+ \mathbb{E}_{\boldsymbol{z}\sim p_{z_i}}\left[\log (1 - D_i(G_i(z)))\right].\label{eq:localvalue}
\end{align}
\begin{figure}[!t]
	\centering
	\includegraphics[width=\columnwidth]{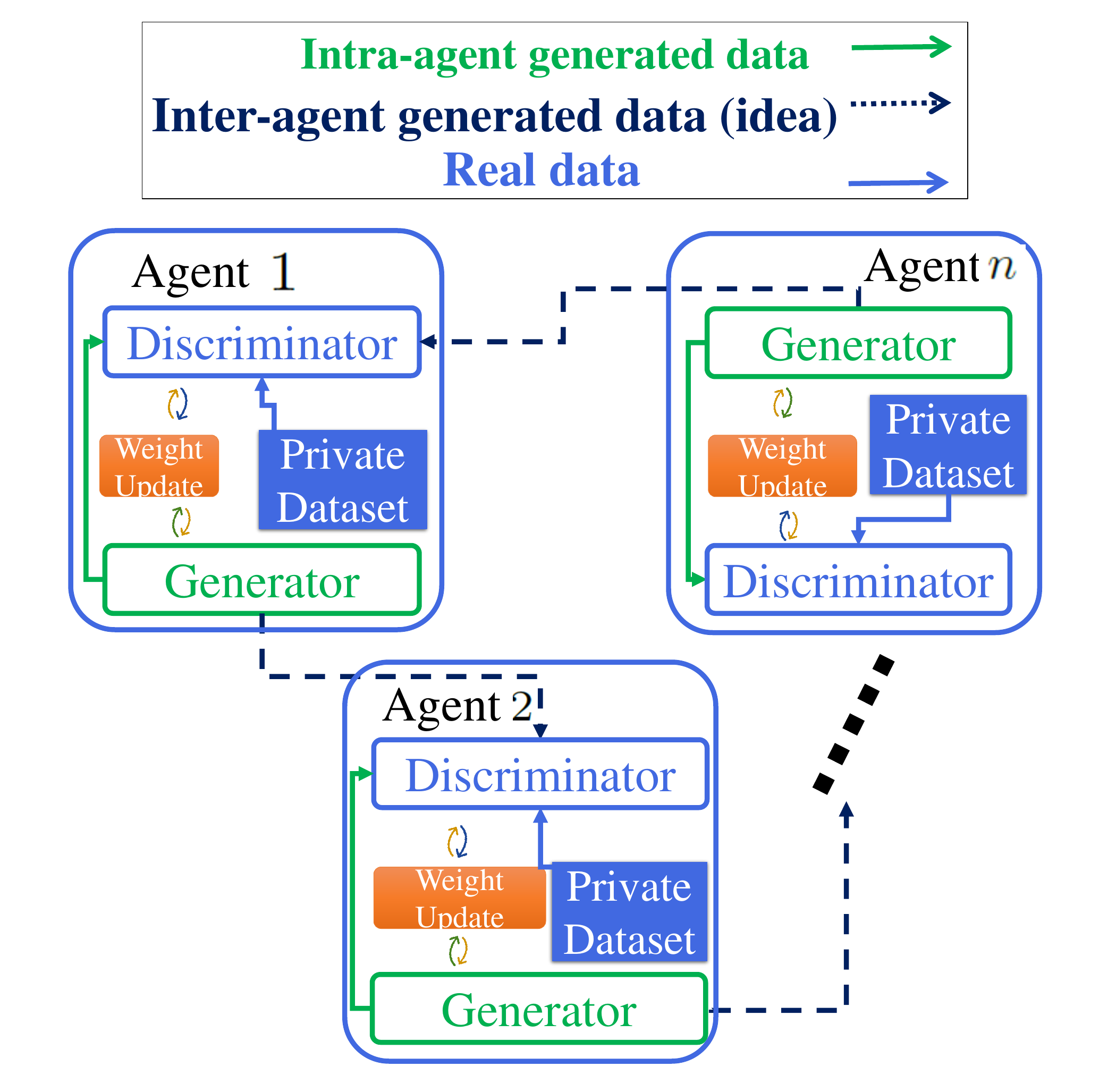}
	\caption{The BGAN architecture.}
	\label{fig:BGAN}
\end{figure}	
In \eqref{eq:localvalue}, the first term forces the discriminator to produce values equal to 1 for the real data. On the other hand, the second term penalizes the data samples generated by the generators. Therefore, the agent's generator aims at minimizing the value function while its discriminator tries to maximize this value. It has been proven in \cite{NIPS2014_5423} that the NE for this game happens when $ p_{g_i} = p_{\text{data}_i} $ and $ D_i = 0.5 $. At the NE, the discriminator cannot distinguish between the generated samples and agent $i$'s real data. Although a standalone GAN can learn the representation of its own dataset, if the owned dataset at each agent is not representative of the entire data space, then the standalone agent will learn a distribution $ p_{g_i}$ that is not exactly the actual data representation. For instance, if an agent has a limited number of data samples that do not span the entire data space, then the learned distribution will be inaccurate \cite{goodfellow2016deep}. In order to cope with this problem, we next introduce BGAN, in which every agent will only share their idea, i.e., their generated points with other neighboring agents without actually sharing their dataset. 

\section{Brainstorming Generative Adversarial Networks Architecture}\label{sec:BGAN}
Let $ \mathcal{N}_i $ be the set of neighboring agents from whom agent $i$ receives ideas, let $ \mathcal{O}_i $ be the neighboring agents to whom agent $i$ sends ideas, and let $ \mathcal{G} $ be the directed graph of connections between the agents as shown in Figure \ref{fig:BGAN}. Here, a neighboring agent for agent $i$ is defined as an agent that is connected to agent $ i $ in the connection graph $ \mathcal{G} $ via a direct link. For our BGAN architecture, we propose to modify the classical GAN value function in \eqref{eq:localvalue} into a \emph{brainstorming value function} which integrates the received generated data samples (ideas) from other agents, as follows:
\begin{align}
	V_i (D_i,G_i,\{G_j\}_{j\in \mathcal{N}_i}) &= \mathbb{E}_{\boldsymbol{x}\sim p_{b_i} }\left[\log D_i(x)\right]\nonumber\\& + \mathbb{E}_{\boldsymbol{z}\sim p_{z_i}}\left[\log (1 - D_i(G_i(z)))\right],\label{eq:Bvalue}
\end{align}
where $ p_{b_i} $ is a mixture distribution of agent $ i $'s owned data and the idea that agent $i$ received from all neighboring agents. Formally, $
p_{b_i} = \pi_i p_{\textrm{data}_i}+ \sum_{j \in \mathcal{N}_i} \pi_{ij}p_{g_j}, $
where $ \pi_i + \sum_{j \in \mathcal{N}_i} \pi_{ij} = 1 $. $ \pi_i $ and $ \pi_{ij} $ represent, respectively, the importance of agent $ i $'s own data and neighbor $ j $'s generated data in the process of brainstorming. Such values can be assigned proportionally to the number of real data samples each agent owns since an agent having more data samples has more information about the data space. From \eqref{eq:Bvalue}, we can see that the brainstorming value functions of all agents in $\mathcal{N}$ will be interdependent. Therefore, in order to find the optimal values for $ G_i $ and $ D_i $, we define a multi-agent game between the discriminators and generators of agents. In this game, the generators collaboratively aim at generating real-like data to fool all of the discriminators while the discriminators try to distinguish between the generated and real data samples of the generators. To this end, we define the total utility function as follows:
\begin{align}\label{eq:totalutility}
	V\left(\left\{D_i\right\}_{i=1}^n,\left\{G_i\right\}_{i=1}^n\right) = \sum_{i=1}^{n}V_i (D_i,G_i,\{G_j\}_{j\in \mathcal{N}_i}).
\end{align}
In our BGAN, the generators aim at minimizing the total utility function defined in \eqref{eq:totalutility}, while the discriminators try to maximize this value. Therefore, the optimal solutions for the discriminators and generators can be derived as follows:
\begin{align}\label{eq:minimax}
	\left\{D_i^*\right\}_{i=1}^n,\left\{G_i^*\right\}_{i=1}^n = \arg\min_{G_1,\dots,G_n}\arg\max_{D_1,\dots,D_n} V,
\end{align}
where for notational simplicity we omit the arguments of $ V\left(D_1,\dots,D_n,G_1,\dots,G_n\right) $. In what follows, we derive the NE for the defined game and characterize the optimal values for the generators and discriminators. At such NE, none of the agents can get a higher value from the game if it changes its generator and discriminator while other agents keeping their NE generator and discriminators.
\begin{proposition}\label{proposition:discriminator}
	For any given set of generators, $ \left\{{G}_1,\dots,{G}_n\right\} $, the optimal discriminator is:
	\begin{align}\label{eq:Dmax}
		D^*_i = \frac{p_{b_i}}{p_{b_i} + p_{g_i}}.
	\end{align}
\end{proposition}
\begin{proof}
	For any given set of generators, $ \left\{{G}_1,\dots,{G}_n\right\} $, we can derive the probability distribution functions for the generators, $ \left\{p_{{g}_1},\dots,p_{{g}_n}\right\} $. Thus, we can write the total utility function as:
	\begin{align}
		V &= \sum_{i=1}^{n}\Bigg[ \int_{x}p_{b_i}(x)\log D_i(x)dx \nonumber\\
		&+ \int_{z}p_{z_i}(z)\log\left(1-D_i\left({G}_{i}\left(z\right)\right)\right)dz\Bigg]\nonumber\\
		&=\sum_{i=1}^{n}\Bigg[ \int_{x}\Big(p_{b_i}(x)\log D_i(x)\nonumber\\
		&+ p_{{g}_i}(x)\log\left(1-D_i\left(x\right)\right)\Big)dx\Bigg],\label{eq:sumval}
	\end{align}
	{where {$p_{g_i}$} is the probability distribution of the generated data for agent $i$ when discriminators reach an optimal state.} Next, to find the maximum of \eqref{eq:sumval} with respect to all of the $ D_i $ values, we can separate every term of the summation in \eqref{eq:sumval} because each term contains $ D_i $ for a single agent $ i $. Thus, the optimal value, $ D^*_i $, is the solution of the following problem:
	\begin{align}
		D_i^* &= \arg\max_{D_i} \int_{x}\Big(p_{b_i}(x)\log D_i(x)\nonumber\\&+ p_{{g}_i}(x)\log\left(1-D_i\left(x\right)\right)\Big)dx.\label{eq:singleD}
	\end{align}
	In order to find the $ D_i $ that maximizes \eqref{eq:singleD}, we can find the value of $ D_i $ that maximizes the integrand of \eqref{eq:singleD}. which is given in \eqref{eq:Dmax}.
\end{proof}

Having found the maximizing values for the discriminators, we can move to the minimization part of \eqref{eq:minimax}. To this end, using \eqref{eq:Dmax}, we can rewrite \eqref{eq:minimax} as follows:
\begin{align}
	G_1^*,\dots,G_n^* = \arg\min_{G_1,\dots,G_n}\underbrace{V(D^*_1,\dots,D^*_n,G_1,\dots,G_n)  }_{W(G_1,\dots,G_n)}.
\end{align}
Now, we can express $ W $ as follows:
\begin{align}
	W &= \sum_{i=1}^{n}\Bigg[ \mathbb{E}_{\boldsymbol{x}\sim p_{b_i}}\left[\log D^*_i(x)\right] \nonumber\\
	&+ \mathbb{E}_{z\sim p_{z_i}}\left[\log\left( 1- D^*_i(G_i(z))\right)\right]\Bigg]\nonumber\\
	&= \sum_{i=1}^{n}\Bigg[ \mathbb{E}_{\boldsymbol{x}\sim p_{b_i}}\left[\log D^*_i(\boldsymbol{x})\right]\nonumber \\
	&+  \mathbb{E}_{x\sim p_{g_j}}\left[\log\left( 1- D^*_i(\boldsymbol{x})\right)\right]\Bigg]\nonumber\\
	&= \sum_{i=1}^{n}\Bigg[ \mathbb{E}_{\boldsymbol{x}\sim p_{b_i}}\left[\log \left(\frac{p_{b_i}}{p_{g_i} + p_{b_i}}\right)\right] \nonumber\\
	&+ \mathbb{E}_{x\sim p_{g_i}}\left[\log\left(\frac{p_{g_i}}{p_{g_i} + p_{b_i}}\right)\right]\Bigg].\label{eq:minimization}
\end{align}

Next, we derive the global minimum of $ W(G_1,\dots,G_n) $.
\begin{theorem}\label{theorem:optimalSigma}
	The global minimum of  $ W(G_1,\dots,G_n) $ can be achieved at the solution of the following equation:
	\begin{align}\label{eq:OptimalSol}
		\left(\boldsymbol{I}-\boldsymbol{B}\right)\boldsymbol{p}_g = \boldsymbol{C}\boldsymbol{P}_{\textrm{data}},
	\end{align}
	where $ \boldsymbol{B} $ is a matrix with element $ \pi_{ij} $ at every row $ i $ and column $ j $, $ \boldsymbol{C} $ is a diagonal matrix whose $ i $-th diagonal element is $ \pi_i $, $ \boldsymbol{p}_g \triangleq \left[p_{g_1},\dots,p_{g_n}\right]^T $, and $ \boldsymbol{P}_{\textrm{data}} \triangleq \left[p_{\textrm{data}_1},\dots,p_{\textrm{data}_n}\right]^T $.
\end{theorem}
\begin{proof}
	We can rewrite $ W $ as follows:
	\begin{align}
		W &= \sum_{i=1}^{n} \mathbb{E}_{\boldsymbol{x}\sim p_{b_i}}\left[\log \left(\frac{p_{b_i}}{p_{g_i} + p_{b_i}}\right)\right] \nonumber \\
		&+  \mathbb{E}_{x\sim p_{g_i}}\left[\log\left(\frac{p_{g_i}}{p_{g_i} + p_{b_i}}\right)\right]\nonumber\\
		& = \sum_{i=1}^{n} \mathbb{E}_{\boldsymbol{x}\sim p_{b_i}}\left[\log \left(\frac{1}{2}\frac{p_{b_i}}{\frac{p_{g_i} + p_{b_i}}{2}}\right)\right] \nonumber \\
		&+  \mathbb{E}_{x\sim p_{g_i}}\left[\log\left(\frac{1}{2}\frac{p_{g_i}}{\frac{p_{g_i} + p_{b_i}}{2}}\right)\right]\nonumber\\
		& = \sum_{i=1}^{n}\left[-\log(4) + \textrm{JSD}(p_{b_i}||p_{g_i})\right]\nonumber\\
		& = - n \log(4) + \sum_{i=1}^{n} \textrm{JSD}(p_{b_i}||p_{g_i}),\label{eq:JSD}
	\end{align}
	where $ \textrm{JSD}(p_{b_i}||p_{g_i}) $ is the JSD between $ p_{b_i}$ and $ p_{g_i} $ with a minimum at $ 0 $. Therefore, we can easily show that $ \sum_{i=1}^{n} \textrm{JSD}(p_{b_i}||p_{g_i})\geq 0 $. Now, given that the minimum of $ W $ occurs when $ p_{g_i} = p_{b_i} $ for $ i \in \mathcal{N} $, we will have: 
	\begin{align}
		p_{g_i} = \pi_i p_{\textrm{data}_i}+ \sum_{j \in \mathcal{N}_i} \pi_{ij}p_{g_j}, \,\, \forall i \in \mathcal{N},
	\end{align}
	which can be simplified to \eqref{eq:OptimalSol} if we move the term $ \sum_{j \in \mathcal{N}_i} \pi_{ij}p_{g_j} $ to the left side of the equation. In this case, we will have $ D_i^* = \frac{p_{g_i}}{2p_{g_i}} = \frac{1}{2} $. Thus, $ W $ will be simplified to:
	\begin{align}
		W &= \sum_{i=1}^{n}\left[ \mathbb{E}_{\boldsymbol{x}\sim p_{b_i}}\left[\log \left(\frac{1}{2}\right)\right] + \mathbb{E}_{x\sim p_{g_i}}\left[\log\left(\frac{1}{2}\right)\right]\right]\nonumber\\
		&= -n\log(4).\label{eq:minimum}
	\end{align}
	By comparing \eqref{eq:JSD} with \eqref{eq:minimum}, we can see that the solution of \eqref{eq:OptimalSol} yields $ \sum_{i=1}^{n} \textrm{JSD}(p_{b_i}||p_{g_i}) = 0 $, thus, minimizes $ W $. 
\end{proof}

Theorem \ref{theorem:optimalSigma} shows that, in order to find the optimal values for $ p_{g_i} $, we need to solve \eqref{eq:OptimalSol}. In the following, we prove that the solution of \eqref{eq:OptimalSol} is unique and is the only minimum of $ W $ and, thus, the game defined in \eqref{eq:minimax} has a unique NE.

\begin{algorithm}[t]
	\caption{BGAN training.}
	\begin{algorithmic}[1]
		\STATE Initialize $ D_i $ and $ G_i $ for $ i \in \mathcal{N} $.
		\STATE \textbf{Repeat:}
		\STATE \quad \textbf{Parallel for} $ i \in \mathcal{N} $:
		\STATE \quad \quad Generate $ b $ samples $ \left\{\boldsymbol{y}_{i}^{(1)},\dots,\boldsymbol{y}_{i}^{(b)}\right\} $ using $ G_i $ and $ p_{z_i} $.
		\STATE \quad \quad \textbf{For} $ j \in \mathcal{O}_i $:  
		\STATE \quad \quad \quad Send $ \pi_{ji}b $ points from the generated samples (ideas), 
		\\ \quad \quad \quad  $ \left\{\boldsymbol{y}_{ji}^{(1)},\dots,\boldsymbol{y}_{ji}^{(\pi_{ji}b)}\right\} $, to agent $ j $.
            \STATE \quad \quad {\textbf{Wait until} ideas from all neighbors are received.}
		\STATE \quad \quad Sample $ \pi_ib $ data samples, $ \left\{\boldsymbol{x}_{i}^{(1)},\dots,\boldsymbol{x}_{i}^{(\pi_ib)}\right\} $,
		\\ \quad \quad from $ \mathcal{D}_i $.
		\STATE \quad \quad Update $ \boldsymbol{\theta}_{d_i} $ by ascending the following gradient:
		\begin{equation}
			\begin{aligned}
				\nabla_{\boldsymbol{\theta}_{d_i}} \frac{1}{b}\Bigg(& \sum_{j \in \mathcal{N}_i}\sum_{k = 1}^{\pi_{ij} b} \log D_i\left(\boldsymbol{y}_{ij}^{(k)}\right) + \sum_{k = 1}^{\pi_{i} b}  \log D_i\left(\boldsymbol{x}_{i}^{(k)}\right)\\ + &\sum_{k=1}^{b}\log(1 - D_i(\boldsymbol{y}_{i}^{(k)}))\Bigg)
			\end{aligned}
		\end{equation}
		\STATE \quad \quad Generate $ b $ samples $ \left\{\boldsymbol{y}_{i}^{(1)},\dots,\boldsymbol{y}_{i}^{(b)}\right\} $ using $ G_i $ \\
		\quad \quad and $ p_{z_i} $.
		\STATE \quad \quad Update $ \boldsymbol{\theta}_{g_i} $ by descending the following gradient:
		\begin{align}
			\nabla_{\boldsymbol{\theta}_{g_i}} \frac{1}{b} \sum_{k=1}^{b}\log(1 - D_i(\boldsymbol{y}_{i}^{(k)}))
		\end{align}
		\STATE \textbf{Until} convergence to the NE
	\end{algorithmic}
	\label{Algorithm:BGAN}
\end{algorithm}

\begin{corollary}\label{cor:uniqueness}
	The defined game between the generators and discriminators of our BGANs has a unique NE where $ \forall i \in \mathcal{N}, D_i = \frac{1}{2} $ and $ p_{g_i} $ is the unique solution of \eqref{eq:OptimalSol}.
\end{corollary}
\begin{proof}
	Since we know $ \pi_i + \sum_{j \in \mathcal{N}_i} \pi_{ij} = 1 $, then $ \boldsymbol{I} - \boldsymbol{B} $ is a diagonally dominant matrix. In other words, for every row of $ \boldsymbol{I} - \boldsymbol{B} $, the magnitude of the diagonal element in a row is larger than or equal to the sum of the magnitudes of all the non-diagonal elements in that row. Therefore, using the Levy-Desplanques theorem \cite{horn2012matrix} we can show that $ \boldsymbol{I} - \boldsymbol{B} $ is a non-singular matrix and, thus, \eqref{eq:OptimalSol} always has a unique solution. Moreover, since the solution of \eqref{eq:OptimalSol} results in having $ \forall i \in \mathcal{N}:\,\, \textrm{JSD}(p_{b_i}||p_{g_i}) = 0 $, that means it can minimize every term of the summation in \eqref{eq:JSD}, thus, it is the unique minimum point of $ W $. Also, using \eqref{eq:Dmax} the solution of \eqref{eq:OptimalSol} yields $ D_i^* = \frac{1}{2} $ which completes the proof.
\end{proof}

{Theorem \ref{theorem:optimalSigma} plays a pivotal role in our approach, as it provides the conditions under which a unique NE can be found for the game defined in our model. This theorem is crucial in understanding the behavior and results of our approach, as it provides the mathematical basis for finding the optimal values for the generator's distributions, which in turn leads to the unique NE.

In essence, Theorem \ref{theorem:optimalSigma} lays the foundation for our model by showing that the global minimum of $ W(G_1,\dots,G_n) $ can be achieved at the solution of equation \eqref{eq:OptimalSol}. This solution represents the optimal values for $ p_{g_i} $, which are the probability distributions for the generators.

By finding these optimal values, we can ensure that the generators produce data samples that are as close as possible to the real data samples. This, in turn, leads to a unique NE for the game, where the generators and discriminators are balanced, and neither can improve their outcome by unilaterally changing their strategy.

Therefore, Theorem \ref{theorem:optimalSigma} is instrumental in achieving our goal of creating a balanced and effective multi-agent system where the agents collaboratively generate real-like data samples.}

{On the other hand,} Corollary \ref{cor:uniqueness} shows that the defined game between the discriminators and generators has a unique NE. At this NE, the agents can find the optimal value for the total utility function defined in \eqref{eq:minimax}. However, one key goal of our proposed BGAN is to show that using the brainstorming approach each generator can integrate the data distribution of the other agents into its generator distribution. To this end, in the following, we prove that in order to derive a generator that is a function of all of the agents' datasets, the graph of connections between the agents, $ \mathcal{G} $, must be \emph{strongly connected}.
\begin{definition}
	An agent $ i $ can reach an agent $ j $ (agent $ j $ is \emph{reachable} from agent $ i $) if there exists a sequence of neighbor agents which starts with $ i $ and ends with $ j $. 
\end{definition}
\begin{definition}
	The graph $ \mathcal{G} $ is called \emph{strongly connected}, if every agent is reachable from every other agent. 
\end{definition}
\begin{theorem}\label{theorem:reach}
	BGAN agents can integrate the real-data distribution of all agents into their generator if their connection graph, $ \mathcal{G} $, is strongly connected.
\end{theorem}
\begin{proof}
	From Corollary \ref{cor:uniqueness}, we know that \eqref{eq:OptimalSol} will always admit a unique solution, $ \boldsymbol{p}_g^* $. In order to derive this solution, we can use the iterative Jacobi method. In this method, for an initial guess on the solution $ \boldsymbol{p}_g^{(0)} $, the solution is obtained iteratively via:
	\begin{align}
		\boldsymbol{p}_g^{(k+1)} = \boldsymbol{I}^{-1}(\boldsymbol{C}\boldsymbol{P}_{\textrm{data}}+\boldsymbol{B}\boldsymbol{p}_g^{(k)}),
	\end{align}
	where $ \boldsymbol{p}_g^{(k)} $ is the $ k $-th approximation on the value of $ \boldsymbol{p}_g^* $. Letting $ \boldsymbol{p}_g^{(0)} = \boldsymbol{0} $, we will have:
	\begin{align}
		\boldsymbol{p}_g^{(k+1)} = (\sum_{m=1}^{k}\boldsymbol{B}^m + \boldsymbol{I})\boldsymbol{C}\boldsymbol{P}_{\textrm{data}}.
	\end{align}
	Therefore, we have:
	\begin{align}
		\boldsymbol{p}_g^* = \lim\limits_{k \rightarrow \infty} \boldsymbol{p}_g^{(k+1)} = (\sum_{m=1}^{\infty}\boldsymbol{B}^m + \boldsymbol{I})\boldsymbol{C}\boldsymbol{P}_{\textrm{data}}.\label{eq:Iter}
	\end{align}
	From \eqref{eq:Iter}, we can see that in order to have $ p_{g_i}^* $ as a function of $ p_{\textrm{data}_j} $, there should be a $ \boldsymbol{B}^m $ whose entry in row $ i $ and column $ j $ is non-zero. This entry of $ \boldsymbol{B}^m $ is non-zero if $ i $ is reachable from $ j $ via $ m $ steps in the graph $ \mathcal{G} $. Therefore, in order to receive information from all of the agent datasets, every agent must be reachable from every other agent which completes the proof.
\end{proof}

Theorem \ref{theorem:reach} shows that, in order to have a BGAN architecture that enables the sharing of data (ideas) between all of the agents, $ \mathcal{G} $ must be strongly connected. In this case, the optimal solution for the generator is a linear mixture of agents real data distribution as follows:
\begin{align}\label{eq:linmix}
	p_{g_i}^* = \sum_{j\in \mathcal{N}}\lambda_{ij}p_{\textrm{data}_j},
\end{align}
where $ \lambda_{ij} $ are some positive-valued constants that can be derived by solving \eqref{eq:OptimalSol}.
However, for the case in which $ \mathcal{G} $ is not strongly connected, the following corollary (whose proof is similar to Theorem 2) shows how the information is shared between the agents.
\begin{corollary}\label{cor:reach}
	A BGAN agent can receive information from every other agent, if it is reachable from that agent.
\end{corollary}
Therefore, in BGANs, the agents can share information about their dataset with every agent which is reachable in $ \mathcal{G} $. In practice, agent $ j $ is reachable from agent $ i $ if there is a \emph{communication path} on the connection graph from agent $ i $ to agent $ j $. Therefore, a BGAN agent can receive information (ideas) from every agent to which it is connected via a communication path. In this case, the optimal generator distribution of each agent $ i $ will be a mixture of the data distributions of the agents which can reach agent $ i $ in graph $ \mathcal{G} $, written as $ p_{g_i}^* = a_i p_{\textrm{data}_i} + \sum_{j \in \mathcal{R}_i} a_{ij} p_{\textrm{data}_j} $. $ \mathcal{R}_i $ is the set of agents that can reach agent $ i $ and $ a_i $ and $ a_{ij} $ come from the solution of \eqref{eq:OptimalSol}.

In order to implement the BGAN architecture, one of the important steps is to integrate the mixture model for each agent in the brainstorming step. For a batch size $ b $ at each training episode, agent $ i $ receives $ \pi_{ij}b $ generated samples from agent $ j \in \mathcal{N}_i $. This approach guarantees that agent $ j $ has $ \pi_{ij} $ contribution in the brainstorming phase compared to other neighbor agents in $ \mathcal{N}_i $. {To ensure all agents are in the same training round and have fresh shared ideas, each agent waits until it has received input from all its neighbors before starting training. This synchronization mechanism ensures concurrent progression through the training rounds, addressing concerns about data freshness.} Algorithm \ref{Algorithm:BGAN} summarizes the steps needed to implement the proposed BGAN architecture. Our BGAN architecture enables a fully distributed learning scheme in which every agent can gain information from all of the other agents without sharing their real data. Next, we showcase the key properties of BGAN by conducting extensive experiments. 

\section{Experiments}\label{sec:exp}
We empirically evaluate the proposed BGAN architecture on data samples drawn from multidimensional distributions as well as image datasets. Our goal here is to show how our BGAN architecture can improve the performance of agents by integrating a brainstorming mechanism compared to standalone GAN scenarios. Moreover, we perform extensive experiments to show the impact of architecture hyperparameters such as the number of agents, number of connections, and DNN architecture. In addition we compare our proposed BGAN architecture with MDGAN, FLGAN, and F2U, in terms of the quality of the generated data as well as the communication resources.


\subsection{Datasets}
For our experiments, we use two types of data samples. The first type which we call the \emph{ring} dataset contains two dimensional samples $ (r\cos\theta,r \sin\theta) $ where $ r \sim \Gamma(\alpha,\beta) $, $ \theta \sim U(0,2\pi) $, and $ \alpha $ and $ \beta $ are chosen differently for multiple experiments. This dataset constitutes a ring shape in two-dimensional space as shown in Figure \ref{fig:ring} and was used since a) the generated points by GAN agents can be visually evaluated, b) the dimensions of each sample has a nonlinear relationship between each other, and c) the stochastic behavior of the data samples is known since they are drawn from a gamma distribution and a uniform distribution and, thus, the JSD between the actual data and the generated data samples can be calculated. We use JSD as the quality measure for the BGAN architecture since, as shown in \cite{NIPS2014_5423} and Theorem \ref{theorem:optimalSigma}, the JSD between the generated data samples and the actual data must be minimized at the NE. The second type of datasets contain the well-known MNIST, fashion MNIST and CIFAR-10 datasets \cite{lecun-mnisthandwrittendigit-2010,xiao2017fashion,krizhevsky2009learning}. These datasets are used to compare BGAN's generated samples with the other distributed GAN architectures such MDGAN, FLGAN, and F2U. 

\begin{figure}
	\centering
	\includegraphics[width=\columnwidth]{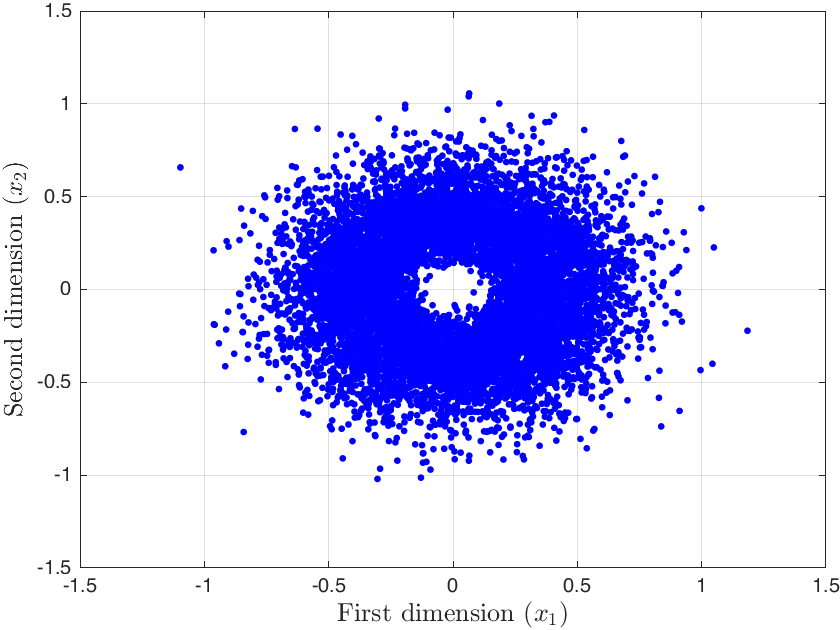}
	\caption{A illustration of the used data samples drawn from a nonlinear combination of gamma and uniform distributions.}
	\label{fig:ring}
\end{figure}  

\subsection{Implementation details}
For every dataset, we use a different DNN architecture depending on the complexity of the dataset. For instance, for the ring dataset we use two simple multi-layer perceptrons with only two dense layers. However, for image datasets that are multi dimensional and are more complex than the ring dataset we use multiple convolutional layers and a dense layer for the discriminator and multiple transposed convolutional layers for the generator. Note that our main goal is to show that our proposed BGAN architecture is fully distributed and communication efficicent. Thus, we do not use sophisticated DNNs as in \cite{radford2015unsupervised,he2016deep,gulrajani2017improved}. However, since the proposed BGAN does not have any restrictions on the GAN architecture, architectures such as those in \cite{radford2015unsupervised,he2016deep,gulrajani2017improved} can naturally be used to achieve higher accuracy in the generated data. In order to train our BGAN, we have used Tensorflow and 8 Tesla P100 GPUs which helped expediting the extensive experiments. We have used multiple batch sizes and the reported values are the ones that had the best performance. Moreover, we have distributed an equal number of samples among the agents and we assign equal values for $ \pi_{ij} $ and $ \pi_i $, unless otherwise stated.

\subsection{Effect of the number of agents and data samples}
GANs can provide better results if they are fed large datasets. Therefore, using the ring dataset, we try to find the minimum number of training samples that is enough for a standalone GAN to learn the distribution of the data samples and achieve a minimum JSD from the dataset. From Figure \ref{fig:numpts}, we can see that the JSD remains constant after 1000 data samples and, hence, in order to showcase the benefits of brainstorming, we will assign less than 1000 samples to each agent to see if brainstorming can reduce the JSD.

For implementing BGAN, we consider a connection graph in which each agent receives data (idea) only from one neighbor and the graph is strongly connected as shown in Figure \ref{fig:graph1}. We implement BGAN with $ 2 $ to $ 10 $ agents with $ 10 $ to $10000$ data samples for each agent. Figure \ref{fig:n2pts} shows the JSD of generated points from the actual dataset for BGAN agents and for a conventional standalone (single) GAN agent for a different number of data samples. Figure \ref{fig:n2pts} shows that, by increasing the number of data samples, the JSD decreases and reaches its minimum value. More importantly, BGAN agents can compensate the lack of data samples that occurs in a standalone GAN by brainstorming with other agents. For instance a standalone agent with 10 data samples has a JSD of 24, however, when the same agent participates in the brainstorming process with 9 other agents (10 agents in total), then, it can achieve a JSD of 13. 

\begin{figure}[t]
	\centering
	\includegraphics[width=\columnwidth]{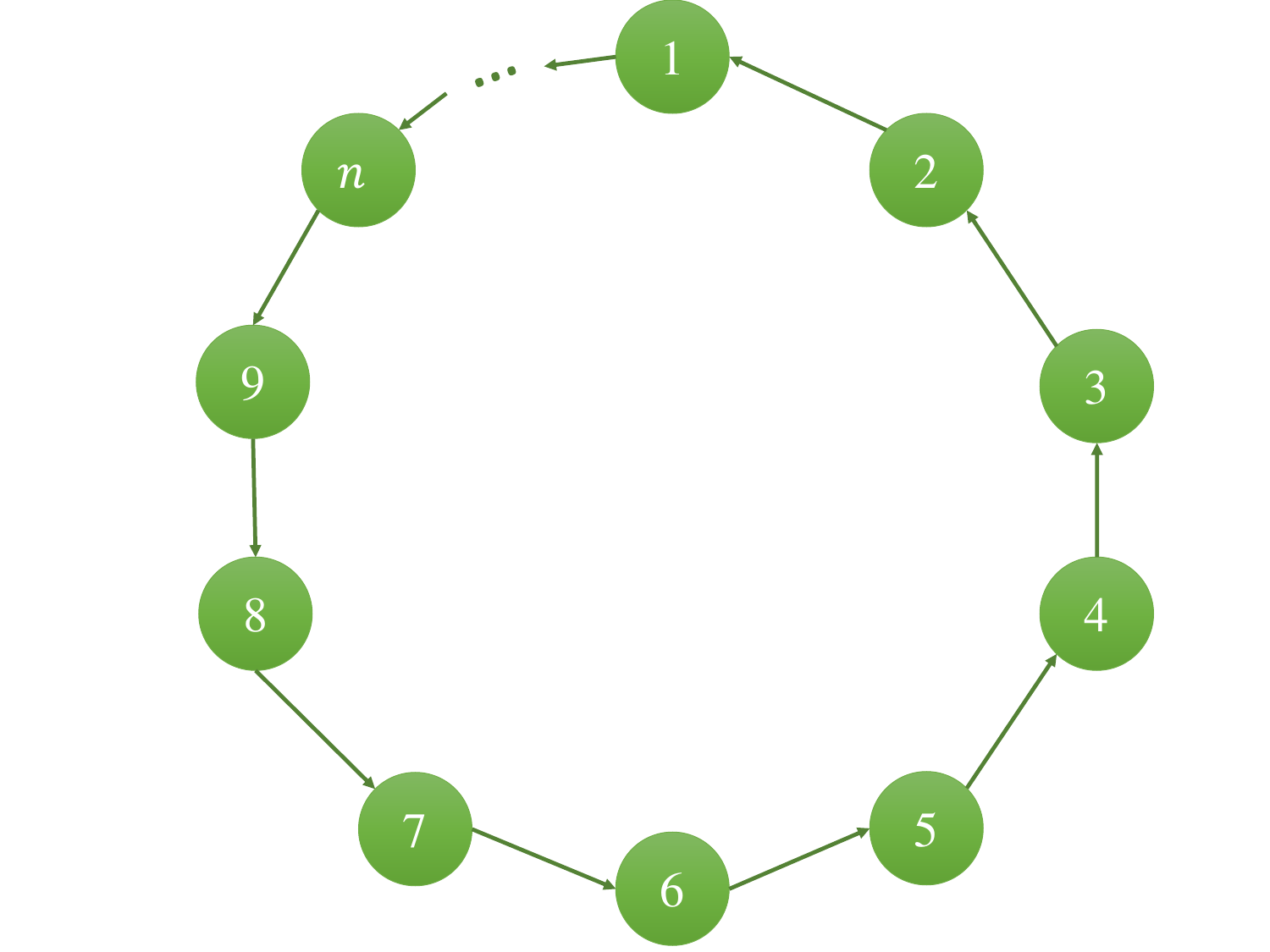}
	\caption{The graph of connections with only one neighbor for each agent that has strong connectivity property.}
	\label{fig:graph1}
\end{figure}  
\begin{figure}[t]
	\centering
	\includegraphics[width=\columnwidth]{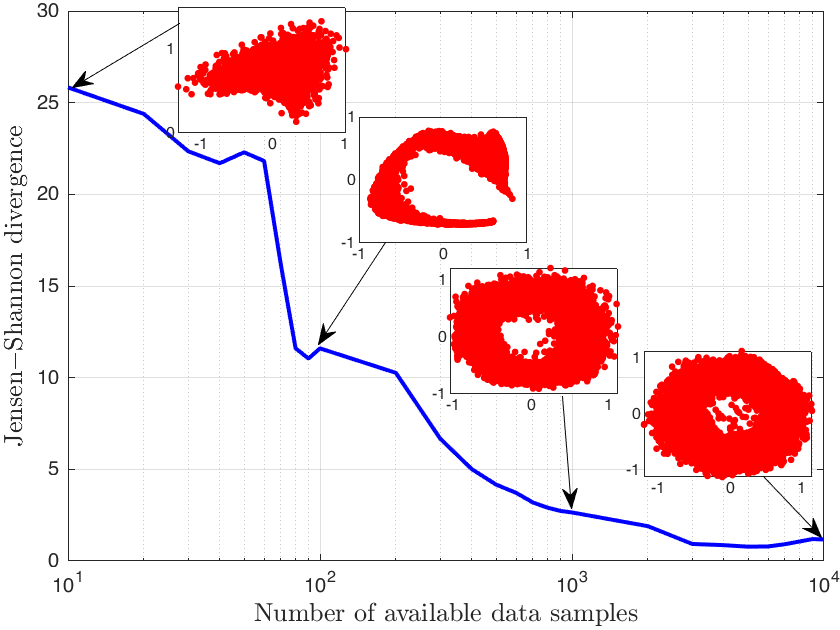}
	\caption{JSD as a function of the number of available samples for a standalone GAN.}
	\label{fig:numpts}
\end{figure}

\begin{figure}[t]
	\centering
	\includegraphics[width=\columnwidth]{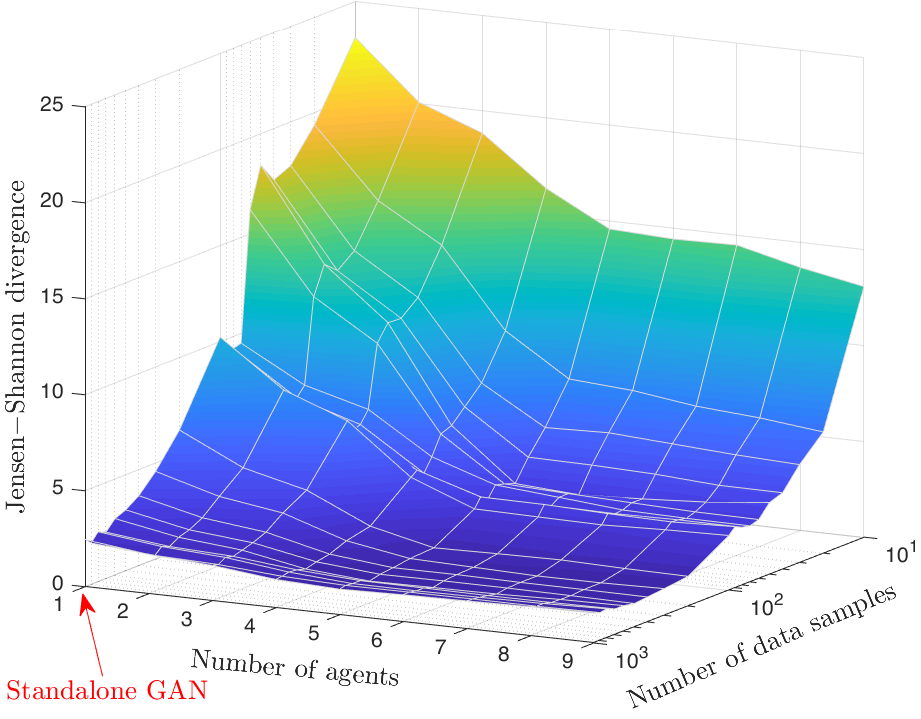}
	\caption{Effect of the number of available samples for each BGAN agent as well as number of agents on the system's JSD.}
	\label{fig:n2pts}
\end{figure} 

Next, in Figure \ref{fig:n2pts_genpts}, we show the generated points by standalone and BGAN agents for different numbers of available data samples. When an agent has access to a small dataset, the GAN parameters will be underfitted, however, brainstorming can still, in some sense, increase the size of dataset by adding the generated samples of neighboring agents into the training set. Therefore, as seen from Figure \ref{fig:n2pts_genpts}, by participating in the brainstorming process, a BGAN agent with a limited dataset can generate data samples that are closer to the actual data distribution in Figure \ref{fig:ring}. This demonstrates that our BGAN architecture can significantly improve the learning and generation performance of agents. 

Note that, no distributed GAN agent can perform better than a standalone GAN that has access to all data samples. The reason that we compare BGAN agents with a standalone GAN agent with the \emph{same} number of data samples is to show how BGAN can help the flow of information between BGAN agents and confirm our theorems. For example, a standalone GAN with 1000 data samples will have a better performance than 10 BGAN agents with 100 data samples each since, in this case, BGAN agents have access to 90\% less real data samples than a standalone GAN, but they try to compensate that lack of data samples by communicating their generated ideas with neighboring agents. Meanwhile, 10 BGAN agents each having 100 data samples have a higher performance on average than 10 standalone GAN agents (who cannot communicate) having 100 samples each. Therefore, a BGAN agent’s average performance is lower bounded by a standalone GAN having the same number of data samples as a single BGAN agent and upper bounded by a standalone GAN having the total number of samples for all BGAN agents. This is a result that has not been shown for other state-of-the-art distributed GANs which is important since it quantifies performance bounds of distributed GAN agents.

\begin{figure}[t]
	\centering
	\includegraphics[width=\columnwidth]{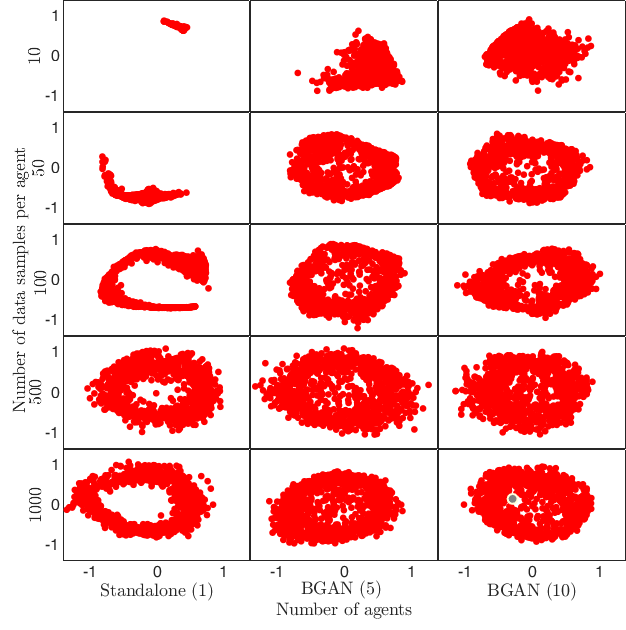}
	\caption{Normalized generated samples of standalone GAN and BGAN with 5 and 10 agents where each agent has access to 10, 50, 100, 500, or 1000 data samples.}
	\label{fig:n2pts_genpts}
\end{figure} 
\subsection{Effect of the connection graph}
In order to study the effect of the connection graph $ \mathcal{G} $ on the performance of BGAN agents we consider two scenarios: 1) A strongly connected graph in which every agent has an equal number of neighbors as shown in Figure \ref{fig:graphn} and 2) a \emph{string} such that all of the agents have only one neighbor except one agent, agent $n$, that does not have any neighbors but sends data to one other agent as shown in Figure \ref{fig:string}.

In the first scenario, we consider 10 agents and we implement the BGAN architecture while varying the number of neighbors for each agent between 1 to 9. We also consider 10, 20, 50, 100, and 1000 data samples for each agent. Figure \ref{fig:con} shows that for cases in which the number of samples is too small (10, 20, and 50) having more neighbors reduces the average JSD for the BGAN agents. However, when the number of samples is large enough adding more neighbors does not affect the JSD. We explain this phenomenon from \eqref{eq:OptimalSol} and \eqref{eq:linmix}, where the optimal generator distribution is shown to be related to the graph structure. For instance, Table \ref{tab:lambda} shows the dependence of each agent's generator distribution on all of the agents' datasets for the cases with 1 and 9 neighbors. From Table \ref{tab:lambda}, we can see that, for the 1-neighbor scenario, as an agent gets farther away from another agent, it will less affect the other agent's generator distribution since $ \lambda_{ij} $ decreases when $ \mod(i-j,10) $ increases. However, for a 9-neighbor scenario, $ \lambda_{ij} $ stays constant for$ \mod(i-j,10)>0 $. Therefore, when the agents have a small dataset and have a higher number of neighbors, they can gain information almost equally from other agents and can span all the data space. As such, as shown in Figure \ref{fig:con}, the JSD of the agent which owns only 10 data samples gets better when the number of neighbors increases. However, when the number of data samples is large then $ p_{\textrm{data}_i} \simeq p_{\textrm{data}} $ for all $ i \in \mathcal{N} $. In this case, $ p_{\textrm{data}_i}\simeq p_{\textrm{data}_i}$ for $ i,j\in \mathcal{N} $ and, thus, from \eqref{eq:linmix} we will have $ p_{g_i} \simeq  p_{\textrm{data}} $ irrespective of the graph structure.

\begin{table}[t!]
	\centering
	\begin{tabularx}{\columnwidth}{c | c | c}
		$ \mod(i-j,10) $ & $\lambda_{ij}$ (1 neighbor)&  $\lambda_{ij}$ (9 neighbors)\\\hline
		0 & $ 0.5005 $ & $ 0.1818 $\\
		1 & $0.2502 $&$ 0.0909 $\\
		2 & $ 0.1251 $&$ 0.0909 $\\
		\vdots & \vdots & \vdots\\
		9 & $0.001 $ & $ 0.0909 $
	\end{tabularx}
	\caption{Effect of graph structure on the generator distribution.}
	\label{tab:lambda}
\end{table}

\begin{figure}[t]
	\centering
	\includegraphics[width=\columnwidth]{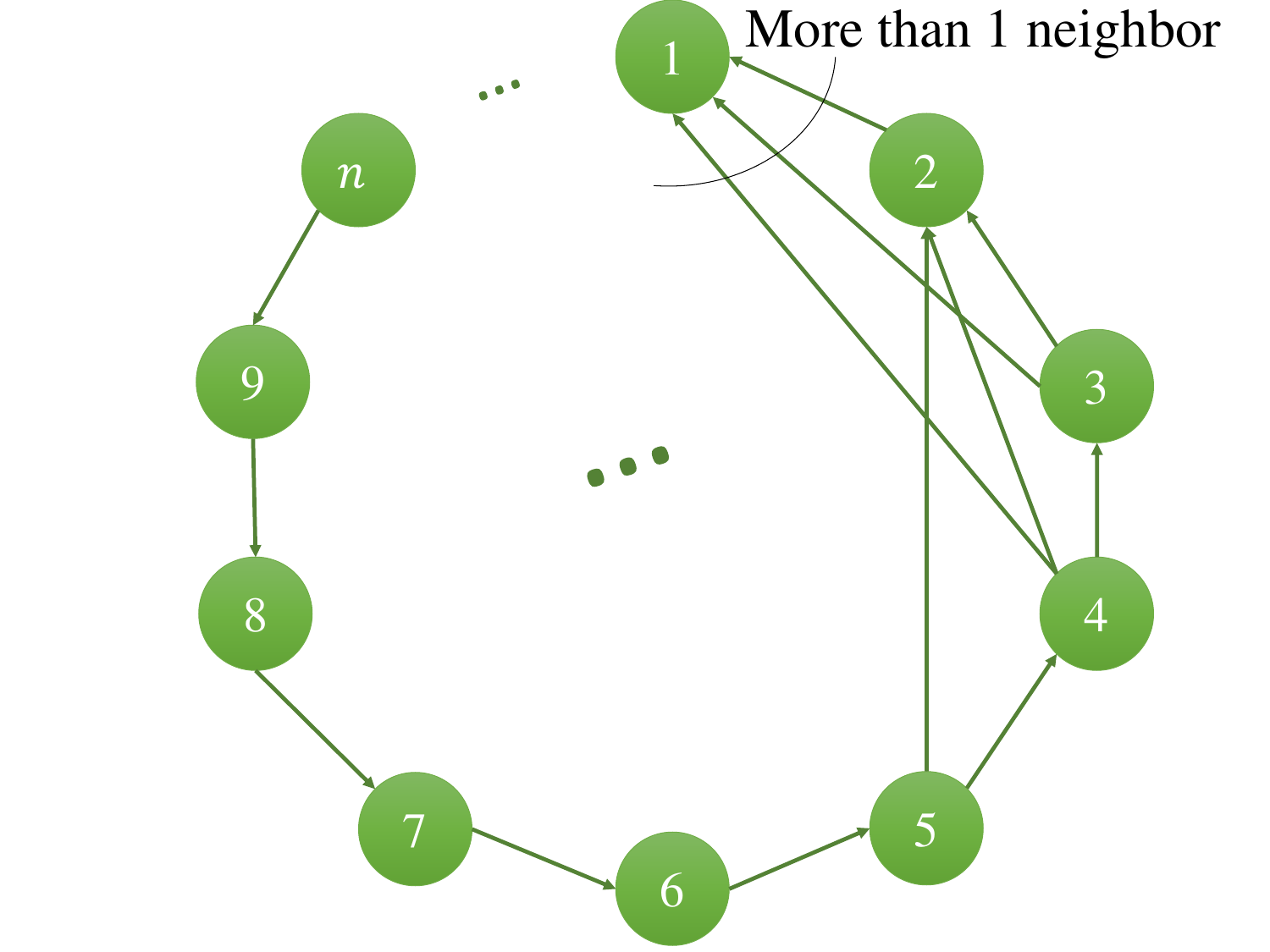}
	\caption{A strongly connected graph with more than 1 neighbor for each agent.}
	\label{fig:graphn}
\end{figure} 
\begin{figure}[t]
	\centering
	\includegraphics[width=0.9\columnwidth]{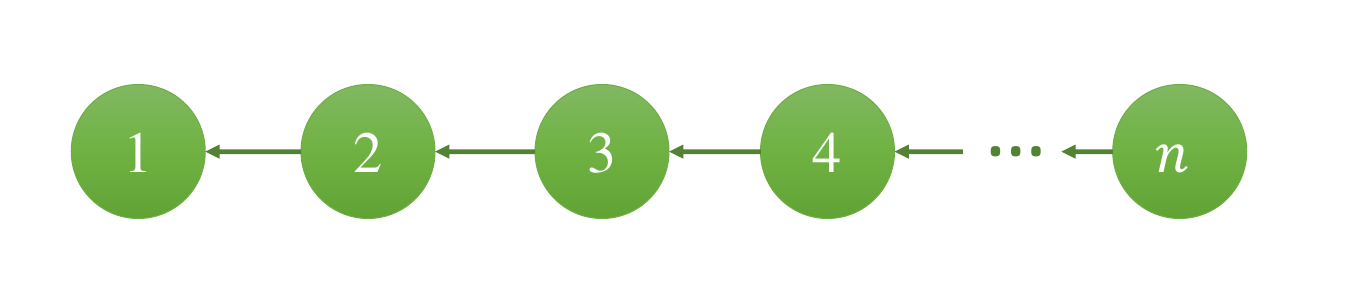}
	\caption{A string graph. All of the agents receive data from only one agent, however, the last agent does not have any neighbors to receive data from.}
	\label{fig:string}
\end{figure} 
\begin{figure}[t]
	\centering
	\includegraphics[width=\columnwidth]{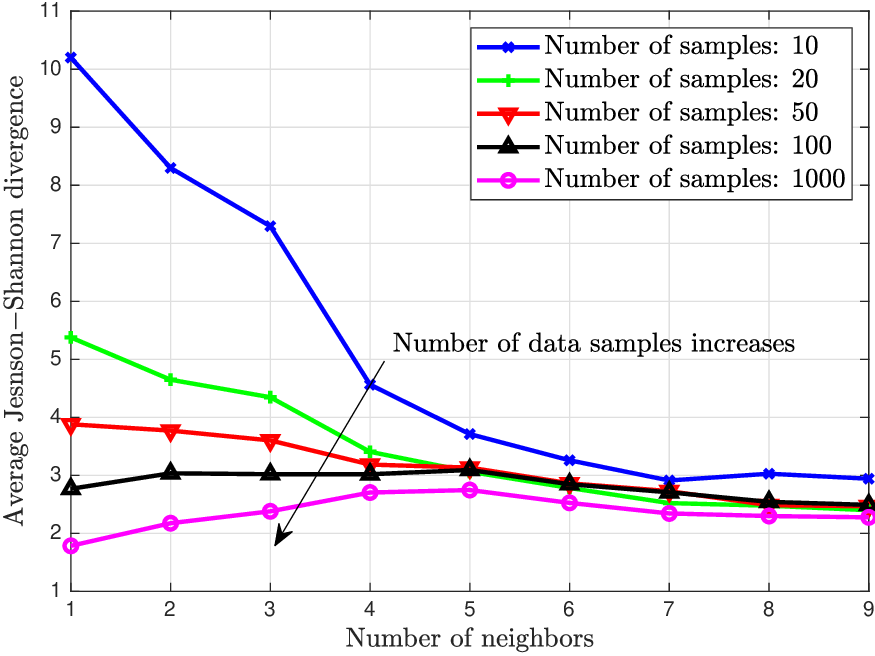}
	\caption{Effect of the number of connections on JSD.}
	\label{fig:con}
\end{figure} 

In the second scenario, we again consider 10 agents. Now, agent 10 sends data to agent 9, agent 9 sends data to agent 8, and so on until agent 2 sends data to agent 1. However, unlike the first scenario, agent 1 does not close the loop and does not send any information to agent 10. Figure \ref{fig:discon} shows the generated samples and JSD of every agent. From Figure \ref{fig:discon}, we can see that the samples generated by agents 1 to 5 are close to the ring dataset and they have a small JSD. However, as we move closer to the end of the string (agent 10), the generated samples diverge from the actual data distribution and the JSD increases. Therefore, in order for each agent to get information from the datasets of all of the agents, the graph should be strongly connected.

\begin{figure}[t]
	\centering
	\includegraphics[width=0.84\columnwidth]{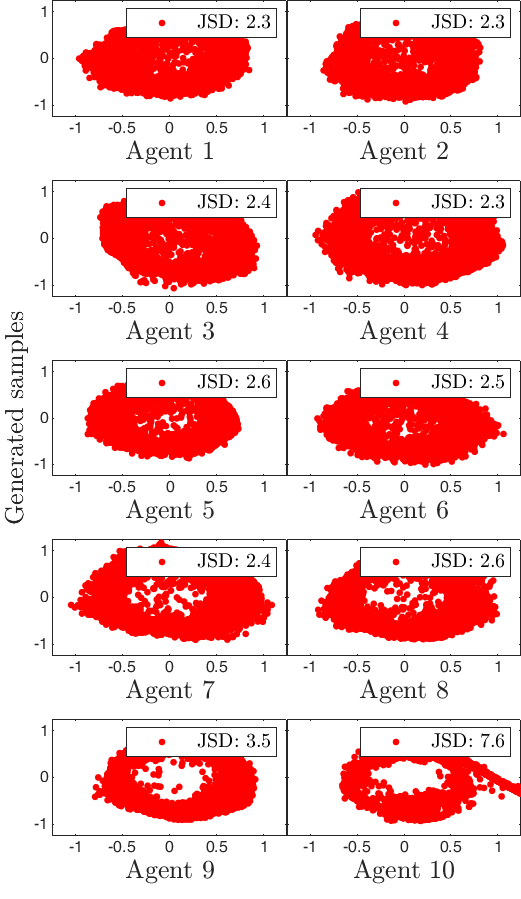}
	\caption{Normalized generated samples of agents of a disconnected BGAN.}
	\label{fig:discon}
\end{figure}

\subsection{Effect of non-overlapping data samples}
Next, we prove that the BGAN agents can learn non-overlapping portions of the other agents' data distributions. To this end, we consider two BGAN agents whereby agent 1 has access to data points only between 0 to $ \pi $ degrees while agent 2 owns the data between $ \pi $ to $ 2\pi $ degrees. Figure \ref{fig:nonoverlap} shows the  available portions of the ring dataset for each agent as well their generated points after brainstorming. As can be seen from Figure \ref{fig:nonoverlap}, the agents can generate points from the ring dataset that they do not own. This showcases the fact that brainstorming helps agents to exchange information between them without sharing their datasets. However, a standalone agent can at best learn to mimic the portion of data that it owns and cannot generate points from the data space that it does not have access to.
\begin{figure}[t]
	\centering
	\includegraphics[width=\columnwidth]{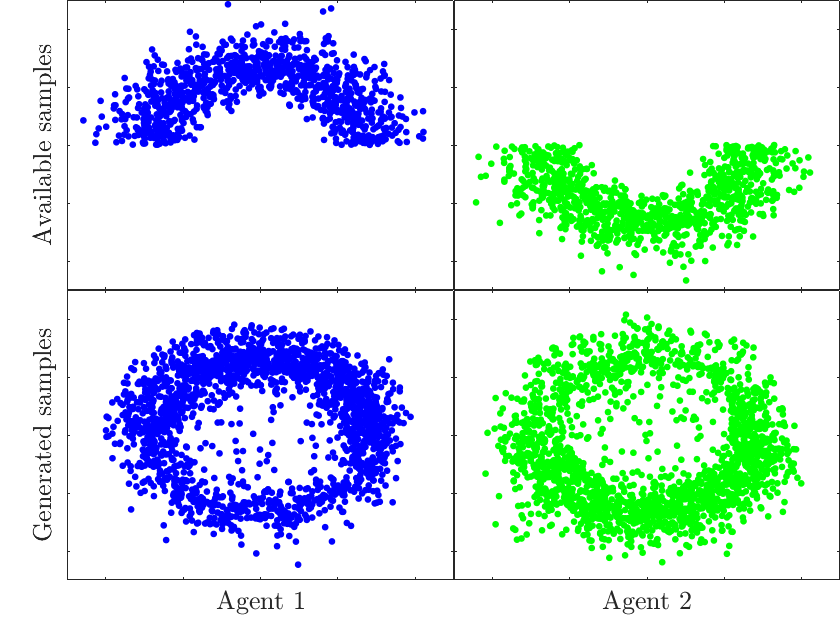}
	\caption{BGAN agents which own nonoverlapping data samples can learn about the other agents' data samples.}
	\label{fig:nonoverlap}
\end{figure} 

\subsection{Effect of different architectures}
We now show how the proposed BGAN can allow the different agents to have a different DNN architecture considering a BGAN with 5 agents whose DNNs differ only in the size of their dense layer. In other words, each agent has a different number of neurons. Moreover, the agents' connection graph is similar to the one in Figure \ref{fig:graphn} with one neighbor. In Figure \ref{fig:diffarch}, we compare the output of the agents resulting from both the standalone and BGAN scenarios. Figure \ref{fig:diffarch} demonstrates that, in the standalone case, agents having denser DNNs will have a lower JSD compared to agents having a smaller number of trainable parameters at the dense layer. However, by participating in brainstorming, all of the GAN agents reduce their JSD and improve the quality of their generated points. This allows agents with lower computational capability to brainstorm with other agents and improve their learned data distribution. Note that this capability is not possible with the other baselines such as FLGAN.
\begin{figure}
	\centering
	\includegraphics[width=0.95\columnwidth]{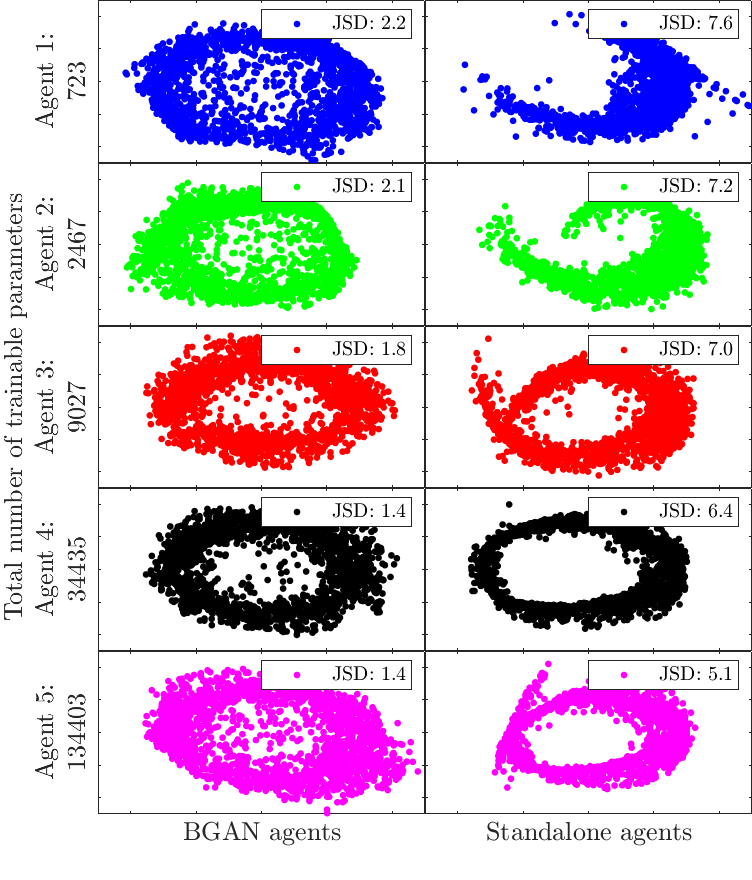}
	\caption{Comparison between the standalone and BGAN agents with different number of trainable parameters.}
	\label{fig:diffarch}
\end{figure}

\subsection{Comparison with FLGAN, MDGAN, and F2U}
In order to compare the generated points by BGAN agents with other state-of-the-art distributed GAN architectures, we use the ring and MNIST dataset. For the ring datasets, we run experiments with 2 to 10 agents where each agent owns only 100 data samples and their connection graph is similar to the one in Figure \ref{fig:graphn} with one neighbor. We compare the average JSD of BGAN agents with the ones resulting from FLGAN, MDGAN, F2U, and standalone GAN agents. We consider $ n\time 100 $ samples for each standalone agent as the upper bound performance of distributed agents and $ 100 $ samples for each standalone agent as the lower bound performance indicator. In other words, no distributed GAN agent can perform better than a standalone GAN that has access to all data samples. On the other hand, distributed architectures should always have a lower JSD compared to a standalone agent with the same number of available data samples.

Figure \ref{fig:jsdcomp} shows the average JSD resulting from the various distributed GAN architectures. We can see from Figure \ref{fig:jsdcomp} that BGAN agents will always have a lower JSD compared to the other distributed GAN agents whereby for the two-agent case, BGAN can achieve a JSD as low as a standalone with 200 data samples. Furthermore, all distributed agents yield a better performance than a standalone agent with 100 samples, however, they cannot achieve a JSD lower than a standalone agent with $ n \times 100 $ samples. In addition, adding more agents reduces the JSD for MDGAN and F2U agents while the JSD of FLGAN and BGAN agents stay constant. For BGAN, we have already seen this fact in Figure \ref{fig:n2pts} where the JSD achieves a minimum value for a particular number of data samples such that adding more agents will not improve the performance of BGAN agents.
\begin{figure}[t]
	\centering
	\includegraphics[width=\columnwidth]{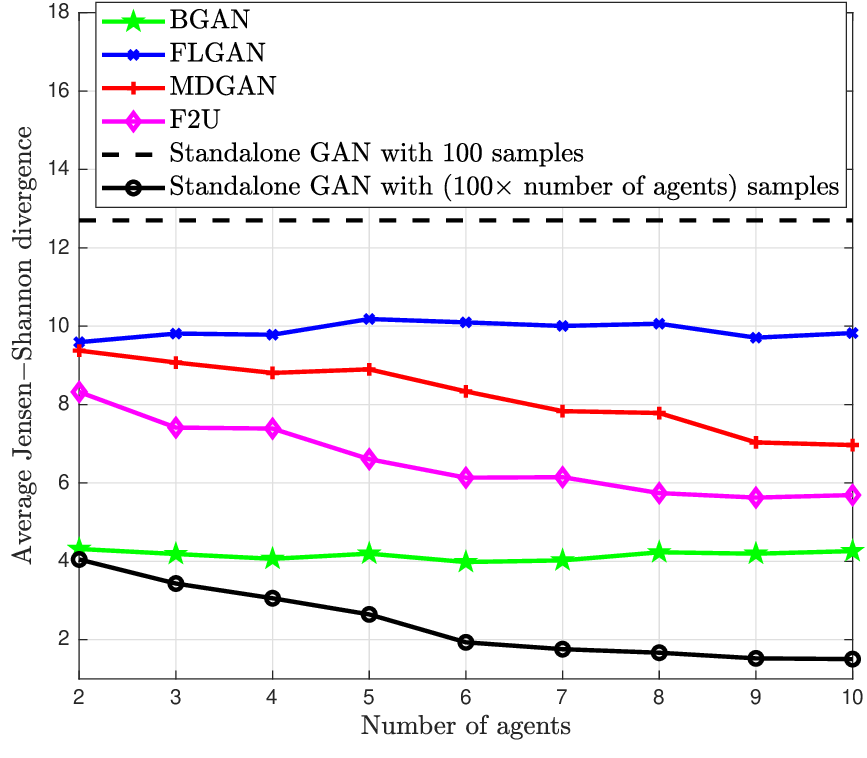}
	\caption{JSD comparison between BGAN, FLGAN, MDGAN, and F2U agents.}
	\label{fig:jsdcomp}
\end{figure} 

\begin{figure*}[t]
	\centering
	\includegraphics[width=0.79\textwidth]{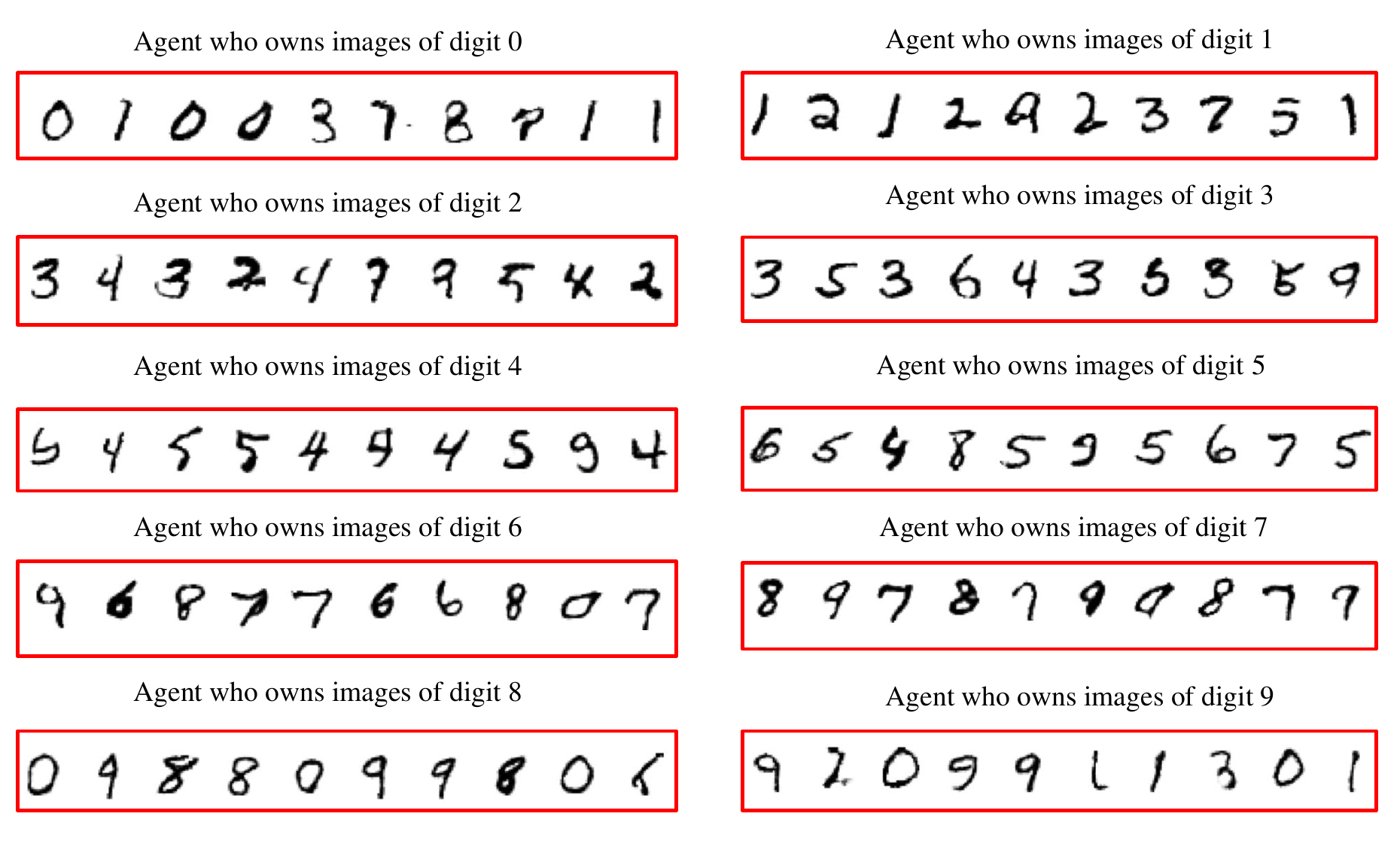}
	\caption{MNIST images generated by BGAN agents.}
	\label{fig:digits}
\end{figure*}


\begin{figure*}[t]
	\centering
	\includegraphics[width=0.79\textwidth]{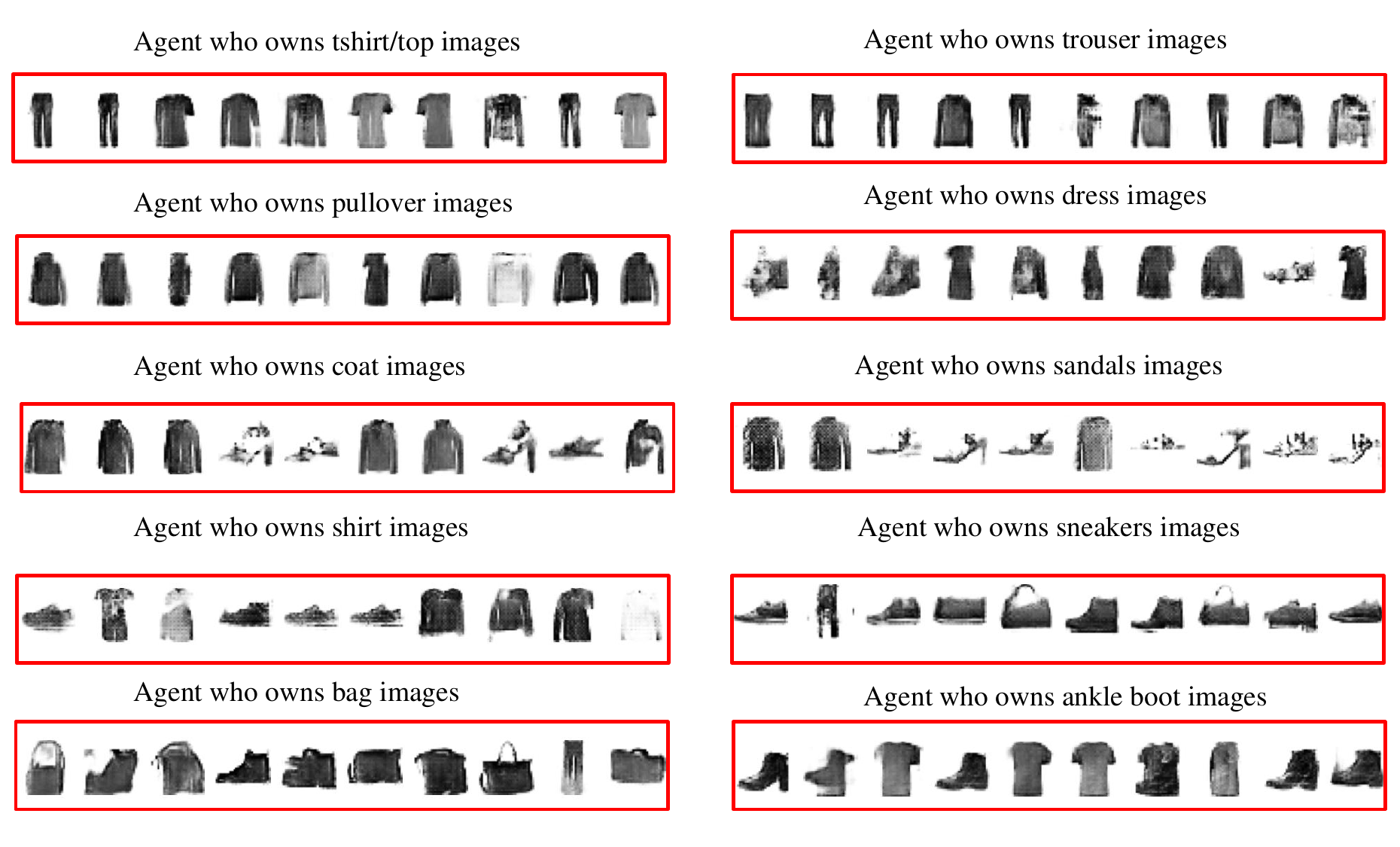}
	\caption{Fashion MNIST images generated by BGAN agents.}
	\label{fig:fashion_MNIST}
\end{figure*}

\begin{figure*}[t]
	\centering
	\includegraphics[width=0.79\textwidth]{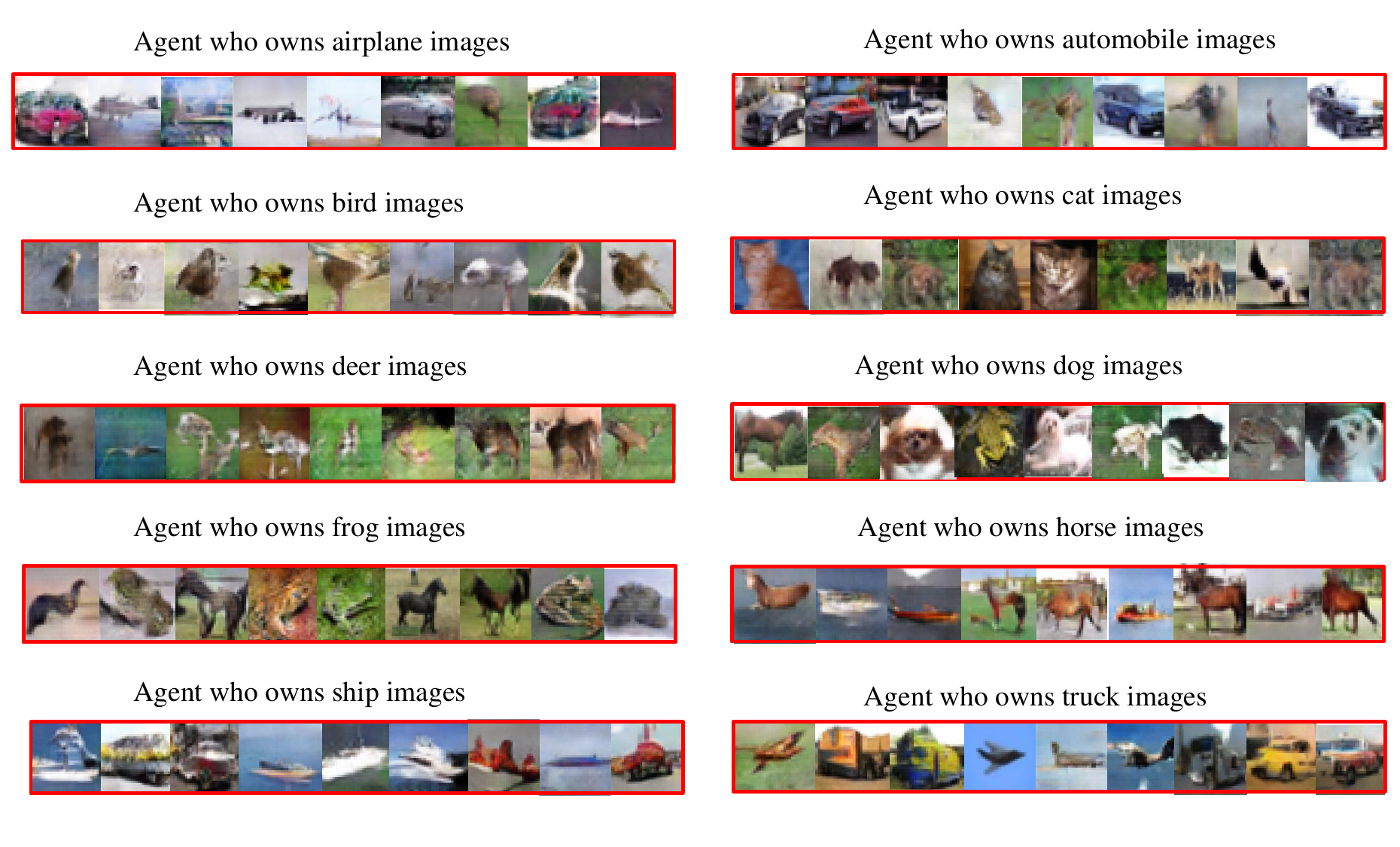}
	\caption{CIFAR-10 images generated by BGAN agents.}
	\label{fig:CIFAR}
\end{figure*}

Furthermore, we compare the performance of BGAN agents with other distributed GANs using the MNIST, fashion MNIST, and CIFAR-10 datasets. In order to show the information flow between the agents, we consider 10 agents each of which owns only images of a single class (a digit in MNIST data, a cloth in fashion MNIST, or an animal type in CIFAR-10 datasets). Figure \ref{fig:digits} shows the images generated by BGAN for the MNIST dataset. From Figure \ref{fig:digits}, we can see that all BGAN agents are able to not only generate digits similar to the dataset that they own, but they can also generate digits similar to their neighbors' datasets. This is a valuable result since, similar to the experiments on the non-overlapping datasets, we can see that BGAN can enable agents to transfer information among them while not sharing their dataset. 

In addition to MNIST dataset, Figures \ref{fig:fashion_MNIST} and \ref{fig:CIFAR} show the generated samples by BGAN agents who have been fed by only fashion MNIST and CIFAR-10 images. In Figures \ref{fig:fashion_MNIST} and \ref{fig:CIFAR} each agent owns the images of only one class. From Figures \ref{fig:fashion_MNIST} and \ref{fig:CIFAR} we can see that using the BGAN architecture, each agent can generate the images that look like its own class as well as its neighbors' classes. These results indicate that the BGAN architecture enables the agents to flow the information regarding their dataset without sharing their real data samples.

Moreover, we calculate the FID between the generated samples of BGAN, MDGAN, FLGAN, and F2U. Table \ref{tab:FID} shows that BGAN outperforms the other architectures in terms of the FID value (normalized with respect to the maximum achieved FID), especially for CIFAR-10 dataset which has more complex images with multiple color channels, the difference between the FID of the BGAN architecture and other architectures is higher. In particular, BGAN has 13\%, 31\%, and 38\% lower FID compared to the best of the other distributed architectures for MNIST, fashion MNIST, and CIFAR-10 datasets, respectively.

\begin{table}[t!]
	\centering
	\begin{tabularx}{\columnwidth}{X|c|c|c|c }
		Dataset & BGAN & MDGAN & FLGAN & F2U\\\hline
		MNIST & $ 1 $ & $ 1.23 $ & $ 1.34 $ & $ 1.13 $\\\hline
		Fashion MNIST& $ 1 $ & $ 1.39 $ & $ 1.47 $ & $1.31$\\\hline
		CIFAR-10 & $ 1 $ & $ 1.45 $ & $ 1.53 $ & $ 1.38 $
	\end{tabularx}
	\caption{FID comparison between BGAN, FLGAN, MDGAN, and F2U agents for MNIST, fashion MNIST, and CIFAR-10 datasets.}
	\label{tab:FID}
\end{table}

{One of the key advantages of BGAN is its low {bandwidth requirements}, particularly for cases with very deep architectures. This advantage becomes especially significant when considering the bandwidth requirements for transmitting data between agents. As shown in Table \ref{tab:com}, the communication requirements of different distributed GAN architectures vary significantly.

In the case of BGAN, each agent receives $ b $ samples from its neighbors at every training epoch, where $ b $ is the batch size. This means that at every time step, $ \mathcal{O}\left(nb|\boldsymbol{x}|\right) $ bandwidth is needed to transmit data between the agents, where $ |\boldsymbol{x}| $ is the size of each data sample.

In contrast, the bandwidth required for MDGAN, FLGAN, and F2U are $\mathcal{O}\left(n\left(b|\boldsymbol{x}| + |\boldsymbol{\theta}_d|\right) \right) $, $\mathcal{O} \left( n\left(|\boldsymbol{\theta}_g|+|\boldsymbol{\theta}_d|\right) \right)$, and $ \mathcal{O}\left(n\left(b|\boldsymbol{x}|+1\right)\right) $, respectively. Clearly, BGAN can significantly reduce the bandwidth overhead of distributed GANs, making it a more efficient choice for distributed learning.

Furthermore, BGANs do not require a central unit that aggregates information from multiple agents, which enhances their robustness in scenarios where an agent fails to communicate with its neighbors. This is a significant advantage over FLGAN, MDGAN, and F2U, which rely on a central unit. In these architectures, any failure in the central unit will disrupt the GAN output at all of the agents.

It is also important to note that our BGAN model does not require the network to be fully connected. Instead, as stated in Theorem \ref{theorem:reach}, the network needs to be strongly connected. A graph is said to be strongly connected if every agent is reachable from every other agent, either directly or indirectly. This ensures the effective propagation of \emph{ideas} or data samples across the network and allows for more flexibility in the network topology. This requirement of strong connectivity is common to other distributed GAN architectures such as FLGAN, MDGAN, and F2U, and it is a crucial factor in managing the bandwidth requirements of the network.}

\begin{table}[t!]
	\centering
	\begin{tabularx}{\columnwidth}{c|X}
		Architecture & Communication resources \\\hline
		BGAN & $ \mathcal{O}\left(nb|\boldsymbol{x}|\right) $\\
		MDGAN & $\mathcal{O}\left(n\left(b|\boldsymbol{x}| + |\boldsymbol{\theta}_d|\right) \right) $\\
		FLGAN & $\mathcal{O} \left( n\left(|\boldsymbol{\theta}_g|+|\boldsymbol{\theta}_d|\right) \right)$  \\
		F2U & $ \mathcal{O}\left(n\left(b|\boldsymbol{x}|+1\right)\right) $
	\end{tabularx}
	\caption{{Bandwidth} requirements of different distributed GAN architectures}
	\label{tab:com}
\end{table}

\section{Conclusion}\label{sec:conc}
{In this paper, we have proposed a novel BGAN architecture that enables agents to learn a data distribution in a distributed fashion in order to generate real-like data samples. We have formulated a game between BGAN agents and derived a unique NE at which agents can integrate data information from other agents without sharing their real data. Our architecture is fully distributed and does not require a central controller, significantly reducing communication overhead compared to other state-of-the-art distributed GAN architectures. Furthermore, our BGAN allows agents with different DNN designs to participate, enabling even computationally limited agents to contribute to brainstorming and gain information from others. Experimental results have shown superior performance in terms of lower JSD and FID across multiple data distributions. Looking forward, potential research directions include exploring different network architectures within the BGAN framework, applying our approach to other types of data such as text or time-series, and investigating more complex multi-agent games in the context of GANs. This work paves the way for future research in distributed learning and data generation, with the potential to inspire new methodologies and applications.}

\def\baselinestretch{1}
\bibliographystyle{IEEEtran}
\bibliography{references}

\end{document}